\documentclass[preprint,12pt]{elsarticle} 

\usepackage{amssymb}

\usepackage{amsmath,amssymb,amsthm}
\usepackage{latexsym}

\usepackage[utf8]{inputenc}
\usepackage{mathtools,mathrsfs,amsfonts,amssymb,amsthm,amsmath,tikz,bbm}
\usepackage{graphics}
\usepackage{multirow}
\usepackage{pgfplots}
\usepackage{subfig}
\usepackage{romannum}
\usepackage{caption}
\usepackage{titlesec}

\titleformat{\subsection}
{\normalfont\normalsize}{\thesubsection}{1em}{}
\titleformat{\subsubsection}
{\normalfont\normalsize}{\thesubsubsection}{1em}{}

\newcommand{\comment}[1]{}

\pgfplotsset{width=8cm,compat=1.9}
\newcommand{\bx}{{\bf x}}
\newcommand{\barbx}{{\bar{\bf x}}}
\newcommand{\by}{{\bf y}}
\newcommand{\barby}{{\bar{\bf y}}}

\newcommand{\bb}{{\bf b}}
\newcommand{\bu}{{\bf u}}
\newcommand{\bh}{{\bf h}}

\newcommand{\A}{\mathcal{A}}

\newcommand{\G}{\mathcal{G}}
\newcommand{\HG}{\mathcal{H}}
\newcommand{\NeigH}{\mathcal{S}}

\newcommand{\T}{\mathcal{T}}
\newcommand{\Trans}{\mathscr{T}}
\newcommand{\R}{\mathbb{R}}

\newcommand{\Neig}{\mathcal{N}}
\newcommand{\ExtNeig}{\widetilde{\mathcal{N}}}
\newcommand{\X}{\mathcal{X}}

\newcommand{\Id}{\text{Id}}
\newcommand{\I}{\mathcal{I}}

\newcommand{\Win}{W^{in}}
\newcommand{\Wout}{W^{out}}

\newtheorem{thm}{Theorem}[section]

\newtheorem{prop}[thm]{Proposition}
\theoremstyle{definition}
\newtheorem{defn}[thm]{Definition}

\theoremstyle{remark}

\usepackage{bbm}

\journal{Neural Networks}

\begin{document} \pagenumbering{arabic}

\begin{frontmatter}



\title{Hypergraph Echo State Network}


\author[1]{Justin Lien}

\affiliation[1]{organization={Mathematical Institute, Graduate School of Sciences, Tohoku University}, 
            addressline={6-3 Aramaki, Aoba Ward}, 
            city={Sendai},
            postcode={980-7}, 
            state={Miyagi},
            country={Japan}}
\begin{abstract}
A hypergraph as a generalization of graphs records higher-order interactions among nodes, yields a more flexible network model, and allows non-linear features for a group of nodes. 
In this article, we propose a hypergraph echo state network (HypergraphESN) as a generalization of graph echo state network (GraphESN) designed for efficient processing of hypergraph-structured data, derive convergence conditions for the algorithm, and discuss its versatility in comparison to GraphESN. 
The numerical experiments on the binary classification tasks demonstrate that HypergraphESN exhibits comparable or superior accuracy performance to GraphESN for hypergraph-structured data, and accuracy increases if more higher-order interactions in a network are identified. 
\end{abstract}


\begin{highlights}
\item A novel HypergraphESN model is proposed to process higher-order interactions in a network.
\item Spectral conditions for the convergence of the HypergraphESN are derived.
\item The non-linear interactions between vertices and hyperedges are introduced in the HypergraphESN model through the incidence graph.
\item The HypergraphESN exhibits a comparable or superior performance compared to the classical GraphESN for hypergraph-structured data.
\item The numerical experiments of binary classification tasks imply that the accuracy improves as more higher-order relationships in a network are identified. 
\end{highlights}


\end{frontmatter}


\section{Introduction} \label{Chap:Intro}
A network refers to a collection of interconnected elements that are linked together to facilitate communication, interaction, or the exchange of information. Graph representation of a network is widely used in a variety of application fields including social sciences, chemistry, transit planning, and epidemiology, where a vertex represents an individual, an entity, an atom, etc, and an edge describes the pairwise relationship between vertices \cite{CHUNG1997, Nowzari2016, Reiser2022, Zhang2016}. However, graph-based modeling is a simplification since it does not adequately represent the realistic case where multiple nodes share common information and interact with others. For example, a benzene ring is a six-carbon cyclic hydrocarbon with alternating single and double bonds between the carbon atoms. 
In a stochastic susceptible-infected-susceptible (SIS) setting for airborne diseases, a healthy individual may be infected when interacting with a group of susceptible individuals. A co-authorship network is a social network that represents collaborations among authors in the production of scholarly publications. On the contrary, the hypergraph exhibits the capacity to capture higher-order relationships among vertices, including intricate interactions between individual vertices and groups of vertices (i.e., hyperedges).
This distinctive feature has spurred the development of hypergraphs in both theory and applications \cite{Bretto2013,Konstantinova2001,Mulas2022}. 

In recent years, the development of machine learning algorithms has been also extended to the hypergraph domain. 
For instance, the hypergraph neural network, built upon the graph neural networks, and its variants have been proposed to handle the unique structural properties of hypergraphs and capture information from hyperedges to perform tasks such as node classification \cite{Feng2019, Jiang2019}. 
Inspired by the graph convolution networks, there have been efforts to develop hypergraph convolution operations to process intricate relationships among nodes \cite{Bandyopadhyay2020,Yadati2019,Yi2020}.
The hypergraph-based modelings potentially provide a more comprehensive understanding of complex relationships within the data, and offer researchers a novel tool for exploring intricate structures and dynamics that may be inadequately represented by traditional graph models \cite{Higham2021,Hirono2021,Jia2021,Zhou2007}.

An echo state network (ESN) \cite{Dai2009,Jaeger2001,Yu2019}, a type of recurrent neural network (RNN), is structured with three fundamental layers: an input layer for receiving external data, a hidden reservoir layer consisting of interconnected non-linear recursive nodes, and an output layer responsible for recording desired outputs. The distinguished feature of ESNs lies in their reservoir, which acts as a dynamic memory, enabling them to capture complex dependency, manage higher dimensional and nonlinear data, and create a rich representation of the input sequence. 
Unlike the traditional RNN, the reservoir is randomly and sparsely initialized and then remains untrained, and only the output layer requires training, leading to computational efficiency and effective performance even with limited training data. The distinctive architectural design of ESNs has led to their popularity and widespread recognition, particularly in applications such as time-series analysis, signal processing, and classification.

Echo state networks have been adapted to address specific challenges for a wide range of applications, giving rise to variants such as leaky echo state networks suitable for slow dynamic systems \cite{JAEGER2007,Lun2019} and echo state Gaussian processes designed for noisy data \cite{Chatzis2011,Huang2019}. 
In this article, we extend the GraphESN \cite{Gallicchio2010} to handle hypergraph-structured data, enabling efficient processing within this domain.
Unlike other types of ESNs, both GraphESN and HypergraphESN incorporate the input network structures for reservoir computing, enhancing their capabilities to handle and identify the complex dynamics within structured data. 
While GraphESN is designed for the linear pairwise relationship between vertices, we may introduce a non-linear bias on the vertex-hyperedge interaction in HypergraphESN before feeding it into the transition function.
Moreover, the convergence condition of the reservoir computing is guaranteed by the contractivity of the transition functions which are closely tied to the spectral properties of the input hypergraph structures.

The article is organized as follows. In section \ref{Chap:Pre}, we review the GraphESN based on \cite{Gallicchio2010} and the basic definition of hypergraphs. In section \ref{Chap:HESN}, we develop the HypergraphESN, as a generalization of GraphESN, provide conditions that guarantee the convergence of algorithms, and discuss the computational complexity. In section \ref{Chap:Experi}, we evaluate our HypergraphESN through numerical experiments on the binary classification tasks and compare it with GraphESN.

\section{Preliminary} \label{Chap:Pre}

In this article, by a network, we mean a collection $V$ of vertices with complicated relationships among them. Each vertex contain information with same types of attributes or features. In practice, graphs and recently, hypergraphs, are typically applied to model the relationships among nodes. Therefore, we start by reviewing the GraphESN and introducing the basic hypergraph theory.

\subsection{A Brief Review of the Graph Echo State Network}

A graph $\G = \big(V(\G),E(\G)\big)$ consists of the vertex set $V(\G)$ and the edge set $E(\G)$ where $e = (v,w) \in E(\G)$ is an un-ordered pair of vertices $v$, $w \in V(\G)$. The adjacency matrix $\A(\G) \in \R^{\left\vert V(G) \right\vert \times \left\vert V(G) \right\vert}$ is a matrix whose diagonal elements are zero and for non-diagonal elements, $\A(\G)_{vw} = 1$ if and only if $(v, w) \in E(\G)$. The neighborhood $\Neig(v)$ of a vertex $v \in V(\G)$ is the set of adjacent vertices, i.e., $\Neig(v) \coloneqq \{ w \in V(\G): (v,w) \in E(\G)\}$. The degree $\deg(v)$ of a vertex $v$ is defined by $\deg(v) \coloneqq \left\vert \Neig(v) \right\vert$. Each vertex $v$ is associated with a column vector $\bu(v) \in \R^{k}$ for some constant $k$ independent of vertices. We say that the column vector $\bu(\G)$, the concatenation of $\{\bu(v)\}_{v \in V(\G)}$, is structured by the graph $\G$ though it does not explicitly contain the information of the edge set. 
In this article, we only consider undirected simple graphs for simplicity (i.e., no self-loop or multi-edge), and column vectors if not explicitly stated. As a reminder, speaking of a graph, we usually mean the pair $\big(\G, \bu(\G)\big)$. Finally, $(\R^k)^\#$ denotes the set of graphs $\G$ such that $\{ \bu(v) \}_{v \in V(\G)} \subset \R^k$ \cite{Gallicchio2010} and with abuse of notation, we also write $\bu(\G) \in (\R^k)^\#$. 

In this study, we consider graph-level tasks with the supervised learning paradigm. The data is of the form $\{ (\G_j,\bu(\G_j),\by_j) \}$ where $\G_j \in (\R^{N_U})^\#$ and $\by_j \in \R^{N_O}$ with $N_U$, $N_O$ independent of $j$. For notational convenience, we drop the subscript $j$. Notice that the input vector $\bu(\G)$ is structured by $\G$ while $\by$ is merely a vector. 

A transduction $\Trans$ is a function from $(\R^k)^\#$ to $(\R^l)^\#$. As a vector can be viewed as a vector on the point graph $(\star,\emptyset)$, a function from $(\R^k)^\#$ to $\R^l$ is also a transduction. The ESN model can be described by a composition of transductions $\Trans = \Trans_{out} \circ \X \circ \Trans_{enc}$ where the encoding function $\Trans_{enc}: (\R^{N_U})^\# \to (\R^{N_R})^\#$ maps the input graph into the structured feature space in which reservoir computing is performed ($N_R$ is the number of units in the reservoir), the state mapping function $\X: (\R^{N_R})^\# \to \R^{N_R}$ computes a vector representation of the internal state, and the output function $\Trans_{out}: \R^{N_R} \to \R^{N_U}$ is a linear function that is the only trainable component.

The encoding function $\Trans_{enc}: \bu \mapsto \bx$ is defined vertex-wise according to the local transition function $\tau_v$ by
\begin{align*}
    \bx(v) &= \tau_v \big( \bu(v), \bx(\Neig(v)) \big) \\
    &= f\big( \Win_v \bu(v) + W_v \bx(\Neig(v)) \big),
\end{align*}
where $f$ is the activation function, the input weight matrix $\Win_v \in \R^{N_R \times N_U}$ is randomly generated, and the reservoir weight matrix $W_v \in \R^{N_R \times (\deg(v) \cdot N_R)}$ is randomly and \textit{sparsely} matrix, and $\bx(\Neig(v))$ is the concatenation of $\{ \bx(w) \}_{w \in \Neig(v)}$. In fact, the notation follows a more general scheme: for a subset $S \subset V(\G)$ and a vector $\bx(\G)$ structured by a graph, $\bx(S)$ denotes the vector obtained by stacking $\{\bx(v)\}_{v \in S}$; a similar notation is adopted for vectors labeled by (hyper-)edges. In practice, the activation function is typically chosen to be a sigmoid or a piecewise linear function like logistic and ReLU (rectified linear unit) activation functions. Throughout the mathematical analysis in this article, we only assume the $1$-Lipschitz continuity and piecewise differentiability of $f$, which covers most of the activation functions in applications. Finally, for a scalar function $f:\R \to \R$, by $f(\bx)$ for a vector $\bx$, we mean component-wise application of $f$. 

If we further adopt the stationary assumption, i.e., $\Win_v = \Win$ and $W_v = [W,\dots,W]$ with $W \in \R^{ N_R \times N_R} $ also called the reservoir weight matrix by abuse of terminology, then the encoding function of the GraphESN is determined by
\begin{align} \label{Eq:GloTrans}
   \bx(\G) &= \tau \big( \G, \bu(\G), \bx(\G) \big) \nonumber \\
   &= f\big( (\Id \otimes \Win) \bu(\G) + (\A(\G) \otimes W) \bx(\G) \big),
\end{align}
where $\tau$ is the (global) transition function.
The well-definedness of the encoding function is not guaranteed. However, if $\tau$ is a contraction, then by the Banach Contraction Principle, there exists a unique solution to the equation (\ref{Eq:GloTrans}). The previous studies of the contraction property either did not relate to the spectral properties of the graph or did assume the null input $\bu(\G) = 0$ (i.e., an autonomous system) \cite{Gallicchio2010, GALLICCHIO2023}. Here, by the explicit formula as above, the contraction property becomes merely a linear algebra question. A similar treatment can be found in \cite{Micheli2022}.

\begin{prop} \label{Prop:GraphESNContraction}
    With the same notation and setup as above, if 
    \[ \left\Vert \A(\G) \right\Vert \left\Vert W \right\Vert < 1, \]
    then $\tau$ is a contraction.
\end{prop}

\begin{proof}
    It follows from a direct computation.
    \begin{align*}
        &\left\Vert \tau \big( \G, \bu(\G), \bx(\G) \big) - \tau \big( \G, \bu(\G), \bx'(\G) \big) \right\Vert \\
    &\le \Vert f\big( (\Id \otimes \Win) \bu(\G) + (\A(\G) \otimes W) \bx(\G) \big) \\ 
    & \indent \indent \indent \indent \indent \indent - f\big( (\Id \otimes \Win) \bu(\G) + (\A(\G) \otimes W) \bx'(\G) \big) \Vert \\
    &\le \left\Vert (\A(\G) \otimes W) ( \bx(\G) - \bx'(\G) ) \right\Vert \\
    &\le \left\Vert \A(\G) \right\Vert  \left\Vert W \right\Vert \left\Vert \bx(\G) - \bx'(\G) \right\Vert.
    \end{align*}
\end{proof}

Hence, to ensure the validity of the contraction principle, the reservoir weight matrix $W$ should satisfy 
\[ ( \max \left\Vert \A(\G) \right\Vert ) \left\Vert W \right\Vert < 1 \]
where the maximum is taken over all input data. As pointed out in \cite{Micheli2022}, this condition is weaker than the one proposed in \cite{Gallicchio2010} as $\left\Vert \A(\G) \right\Vert \le \left\vert V(\G) \right\vert$ where the equality occurs if and only if the graph is complete. 

Because the number of vertices may vary from graph to graph, the state mapping function $\X: \R^{N_R \cdot \left\vert V(\G) \right\vert} \to \R^{N_R}$ is introduced to find a representation of the internal state $\bx(\G)$ (or graph $\G$). There are several choices of the state mapping function. One of them is the mean state mapping $\X_{MSM}$ given by
\[ \X_{MSM}: \bx(\G) \mapsto \barbx(\G) = \frac{1}{\left\vert V(\G) \right\vert} \sum_{v \in V(\G)} \bx(v). \]

Finally, the output function $\Trans_{out}: \R^{N_R} \to \R^{N_O}$ is a linear function with or without the bias term given by
\[ \barby(\G) = \Wout \barbx(\G) \]
or
\[ \barby(\G) = \Wout \barbx(\G) + \bb \]
where the readout weight matrix $\Wout \in \R^{N_O \times N_R}$ and the bias vector $\bb \in \R^{N_O}$ are trained by linear models (e.g., ridge regression, support vector machines). 

\subsection{An Introduction to Hypergraphs}

\begin{defn}
    A hypergraph $\HG$ consists of $V(\HG)$, the set of vertices, and $H(\HG)$, a collection of non-empty subsets of $V(\HG)$. Each $h \in H(\HG)$ is called a hyperedge. 
\end{defn}

\begin{defn}
    The degree of a vertex $v \in V(\HG)$ is the number of hyperedges containing $v$. That is, $\deg(v) = \left\vert \{h \in H: v \in h \} \right\vert$. The degree (or cardinality, the size) of a hyperedge $h \in H(\HG)$ is the number of vertices it contains; i.e., $\deg(h) \coloneqq \left\vert h \right\vert$. We usually informally say $h$ is a small (large) hyperedge if its degree is small (large, respectively).
\end{defn}

By the definition of degrees, we have 
\[ \left\vert V(\HG) \right\vert <\deg(v)> = \left\vert H(\HG) \right\vert <\deg(h)>, \]
where $< \cdot >$ denotes the average. 

As each hyperedge may contain multiple vertices, we define the notion of \textit{neighborhood} in a detailed manner.
\begin{defn}
    The open neighborhood $\Neig(v,h)$ of a vertex $v$ in a hyperedge $h$ containing $v$ is given by $\Neig(v,h) = \{ w \ne v \in h \}$. The closed neighborhood $\ExtNeig(v,h)$ of $v$ in a hyperedge $h$ containing $v$ is defined by $\ExtNeig(v,h) = \Neig(v,h) \cup \{ v \}$. For a vertex $v \in V(\HG)$, the set of hyperedges containing $v$ is given by $\NeigH(v) = \{ h \in H(\HG): \: v \in h \}$. The open neighborhood $\Neig(v)$ of $v$ is given by $\Neig(v) = \bigcup_{h \in \NeigH(v)} \Neig(v,h)$ and the closed neighborhood of $v$ is given by $\ExtNeig(v) = \Neig(v) \cup \{ v \}$. 
\end{defn}

The structure of the hypergraph $\HG$ can be encoded into the incidence matrix, which we now define.

\begin{defn}
The incidence matrix $\I \in \R^{\left\vert V(\HG) \right\vert \times \left\vert H(\HG) \right\vert}$ is given by
$$ \I_{vh} \coloneqq
\begin{cases}
1 & \text{ if } v \in h\\
0 & \text{otherwise.}
\end{cases} $$
\end{defn}

As in the graph theory, we define the degree matrix and adjacency matrix.

\begin{defn}
    The degree matrix $D \in \R^{\left\vert V(\HG) \right\vert \times \left\vert V(\HG) \right\vert}$ is a diagonal matrix given by
    $$ D_{vw} \coloneqq
    \begin{cases}
    \deg(v) & \text{ if } v=w\\
    0 & \text{otherwise.}
    \end{cases} $$
\end{defn}

\begin{defn}
    The adjacency matrix $A \in \R^{\left\vert V(\HG) \right\vert \times \left\vert V(\HG) \right\vert}$ is given by, for $v \ne w$,
    \[ A_{vw} \coloneqq \left\vert \{h \in H: \text{$v$, $w\in h$} \} \right\vert \]
    and $A_{vv} = 0$. That is, $A_{vw}$ is the number of hyperedges containing distinct $v$ and $w \in V$.
\end{defn}

By the definition of incidence, adjacency, and degree matrices, we have
\begin{align} \label{Eq:IncAdjDeg}
    \I\I^T = A + D.
\end{align} 

\begin{defn}
    The hypergraph $\HG$ is connected if for every pair of vertices $v$, $w$, there exists a path connecting $v$ and $w$. That is, there exist $v = v_1,...,v_k = w$ and $h_1,...,h_{k-1}$ such that $\{ v_i,v_{i+1}\} \subset h_i$ for each $i = 1,...,k-1$.
\end{defn}

There are many approaches to check the connectedness of a hypergraph. One of the easiest methods is the following characterization: the hypergraph is connected if and only if $A^{\left\vert V(\HG) \right\vert}$ does not have a non-zero entry.
In what follows, we always assume that the hypergraph $\HG$ is connected and does not have degenerate hyperedge, i.e., $\deg(h) = 1$, or multi-hyperedge, i.e., $h = h'$ for some distinct $h$, $h' \in H(\HG)$.

Finally, we provide a toy example of hypergraph $\HG = \big( V(\HG), H(\HG) \big)$ as shown in Figure \ref{Fig:1} where
\begin{align*}
    V(\HG) &= \{v_1,\dots,v_6\} \\
    H(\HG) &= \big\{h_1 = \{v_1,v_2,v_3,v_4\},h_2 = \{v_1,v_2,v_5\},h_3=\{v_1,v_6\} \big\}.
\end{align*}
The incidence, degree, and adjacency matrices and be written as

\begin{align*}
\hspace{-17pt}
    \I(\HG) = 
    \begin{bmatrix}
        1 & 1 & 1 \\
        1 & 1 & 0 \\
        1 & 0 & 0 \\
        1 & 0 & 0 \\
        0 & 1 & 0 \\
        0 & 0 & 1 \\
    \end{bmatrix}
    , \: 
    D(\HG) = 
    \begin{bmatrix}
        3 & 0 & 0 & 0 & 0 & 0 \\
        0 & 2 & 0 & 0 & 0 & 0 \\
        0 & 0 & 1 & 0 & 0 & 0 \\
        0 & 0 & 0 & 1 & 0 & 0 \\
        0 & 0 & 0 & 0 & 1 & 0 \\
        0 & 0 & 0 & 0 & 0 & 1 \\
    \end{bmatrix}
    ,\: \text{and }
    \A(\HG) = 
    \begin{bmatrix}
        0 & 2 & 1 & 1 & 1 & 1 \\
        2 & 0 & 1 & 1 & 1 & 0 \\
        1 & 1 & 0 & 1 & 0 & 0 \\
        1 & 1 & 1 & 0 & 0 & 0 \\
        1 & 1 & 0 & 0 & 0 & 0 \\
        1 & 0 & 0 & 0 & 0 & 0 \\
    \end{bmatrix}
    ,
\end{align*}
respectively.

\begin{figure}[!ht]
    \centering  \includegraphics{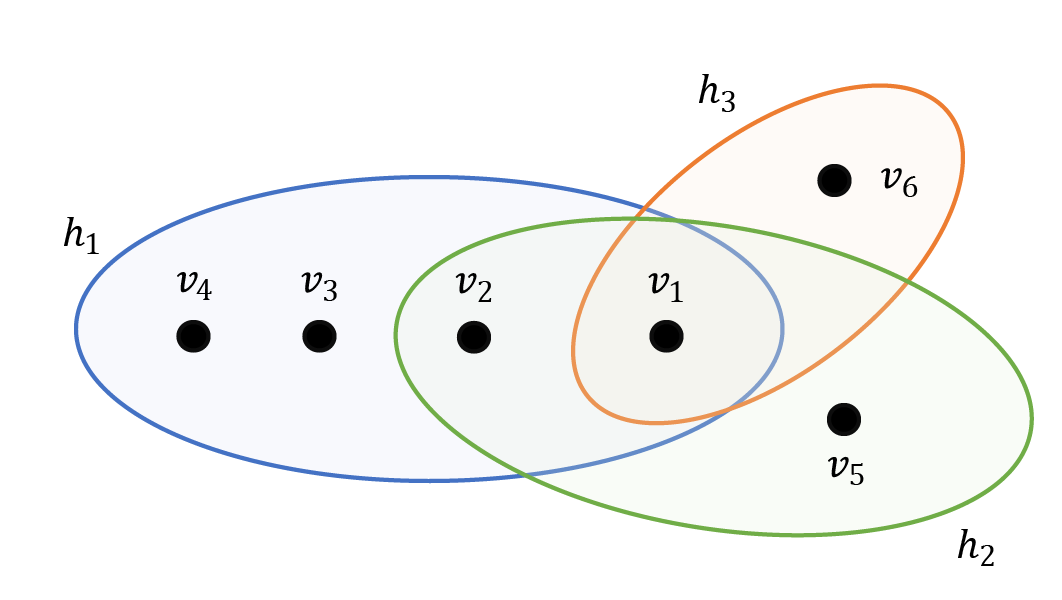}
    \caption{An example of a hypergraph. Each block dot represents a node and each colored elliptic indicates a hyperedge.} 
    \label{Fig:1}
\end{figure}

\section{The Hypergraph Echo State Network Model} \label{Chap:HESN}
\subsection{Ideas}

In the hypergraph domain, a node $v$ should not be directly influenced by the adjacent vertices but by \textit{its incident hyperedges}, and the contribution of a hyperedge $h \in \NeigH(v)$ stems from all nodes insides $h$ including or except $v$ (i.e., a vertex may influence itself or not). The nodes can contribute to the influence of a hyperedge in many different ways. One of the simplest ways is to assume that the contribution of nodes to a hyperedge is a function of \textit{the sum of nodes}.

Let's consider an hypergraph SIS model for example \cite{Higham2021}. A node $v$ is an individual and a hypergraph $h$ is a group of individuals like a class, an office room, etc; $v \in h$ if and only the individual $v$ belongs to the group $h$. During an epidemic of airborne disease like COVID-19, a healthy individual may experience a higher infection rate if more people in the group are infected (instead of a specific person in the group), and if the individual belongs to many groups at the same time \cite{Higham2021,Wilsonm2020}.

Moreover, in the ordinary GraphESN, each adjacent vertex influences $v$ equally.
However, in a real-world network, some of the adjacent vertices, say $v$ and $w$, may belong to multiple groups at the same time and hence, should share more information. In the SIS example, it means that the possibility of the node $w$ infecting $v$ is higher than that of other nodes. Such a feature is not taken into consideration in the original GraphESN. In our HypergraphESN, as the influence comes from the hyperedges and two vertices may belong to multiple hyperedges at the same time, informally speaking, the hypergraph model not only records the higher-order relationships but takes the \textit{closeness} of two vertices into consideration.

\subsection{Mathematical Description}

The architecture of the HypergraphESN consists of an input layer of $N_U$ units, a hidden layer of $N_R$ non-linear recursive units (the reservoir), and an output layer of $N_O$ units as the GraphESN does. The given data is of the form $\big\{\big( \HG_j,\bu(\HG_j)),\by_j \big) \big\}_{j = 1}^{N}$ where $\HG_j = \big(V(\HG_j),H(\HG_j)\big) \in (\R^{N_U})^\#$ is the hypergraph, $\bu(\HG_j)$ is the hypergraph-structured input vector, $\by_j \in \R^{N_O}$ is the (unstructured) target vector. We denote by $\bx(\HG_j) \in \R^{N_R \cdot \left\vert V(\HG_j )\right\vert}$ the (structured) internal state. For notational convenience, we drop the subscript $j$. The transductions connecting the core architectures in HypergraphESN and GraphESN are similar: the encoding function, the state mapping function, and the output function. However, the encoding processing in HypergraphESN is modified to handle the hypergraph structure. Here, information on vertices is aggregated along the incident hyperedges and subsequently fed back to the vertices themselves, which effectively divides the transition function into two steps. Therefore, in what follows, we focus on the encoding function. 

\subsubsection{Model 1}

As in the GraphESN, we start with a local transition function and then stack them together to obtain the global transition function. 
In Model 1, for each vertex $v$, the internal states on the \textit{open neighborhood} of its incident hyperedges are collected to compute the internal states on incident hyperedges, which are later used to update the internal state at $v$ in the iterative scheme.

Let $(\R^k)^\star$ denote the set of hyperedges with degree $k \ge 2$. For a vertex $v$, the function $\T_{k,v}: (\R^{k})^\star \to \R^{N_R}$ maps the internal states on an incident hyperedge to a hyperedge-based vector. 
For notational convenience, we symbolically write $\T = \{ \T_{k,v} \}$ and call them the aggregation functions, in agreement with the terminology in \cite{Yi2020}. 
In this article, we assume $\T_{k,v}$ to be of the form, for $h \in \NeigH(v)$,
\[ \T_{k,v}(h) \coloneqq g\big( \sum_{ w \in \Neig(v,h) } \bx(w) \big), \]
where $g$ is a Lipschitz function (also called the aggregation function by abuse of notation) independent of the input data, with Lipschitz constant $L_g$. By staking $\{ \T_{\deg (h),v} (h) \}_{h \in \NeigH(v)}$, we obtained a local hyperedge-based internal state $\bh(\Neig(v)) \in \R^{N_R \cdot \deg(v)}$.
The encoding function is defined vertex-wise by the local transition function $\tau_v$ by 
\begin{align} \label{Eq:LocTrans2}
    \bx(v) &= \tau_v \big( \bu(v), \bh(\Neig(v)) \big) \nonumber \\
    &= f\big(\Win_v \bu(v) + W_v \bh(\Neig(v))\big).
\end{align}

Staking all local transition functions, we see that the encoding function $\Trans_{enc}: \bu(\HG) \mapsto \bx(\HG)$ is determined by the global transition function $\tau$ which is the concatenation of the local transition functions,
\[ \bx(\HG) = \tau\big(\HG,\bu(\HG),\bx(\HG)\big). \]
To ensure the uniqueness and the existence of the solution, we would like to provide conditions so that the Banach Contraction Principle implies the convergence of the iterative scheme 
\[ \bx_t(\HG) = \tau\big(\HG,\bu(\HG),\bx_{t-1}(\HG)\big), \]
to the unique solution. In practice, the initial state is set to be the null state $\bx_{0}(\HG) = 0$ and the iteration ends whenever $\left\Vert 
\bx_{t} - \bx_{t-1} \right\Vert < \epsilon$ for some prescribed threshold $\epsilon$.

To derive conditions for contractivity, we assume the stationary assumption and find an explicit form of the local transition function $\tau_v$ whose $k$-th component is denoted by $\tau_{v,k}$. We refer to the subscript $\{ v, k \}$ as the graph and vectorial component of the global (or local) transition function. 
\begin{align*}
    \tau_{v,k} &= e^T_{v,k} f\big(\Win \bu(v) + W \sum_{h \in \NeigH(v)} \T_{\deg(h),v} (\bx(h))\big) \\
    &= e^T_{v,k} f\big(\Win \bu(v) + W \sum_{h \in \NeigH(v)} g(\sum_{ v \in \Neig(v,h)} \bx(w) )\big) \\
    &= e^T_{v,k} f\big(\Win \bu(v) + W \sum_{h \in \NeigH(v)} g(\sum_{ w \in V(\HG)} ( \I_{wh} - \I_{wh}\delta_{vw}) \bx(w) )\big) \\
    &= e^T_{v,k} f\big(\Win \bu(v) + W \sum_{h \in H(\HG)} \I_{vh} g(\sum_{ w \in V(\HG)} ( \I_{wh} - \I_{wh}\delta_{vw}) \bx(w) )\big).
\end{align*}
By $e_{v,k} \in \R^{\left\vert V(\HG) \right\vert \cdot N_R}$, we mean a column vector whose $\{v,k\}$-entry is $1$, and $0$ otherwise. If $g = \Id$, then by (\ref{Eq:IncAdjDeg}) we have
\begin{align*}
    \tau_{v,k} &= e^T_{v,k} f\big(\Win \bu(v) + W \sum_{ w \in V(\HG)} \A_{vw} \bx(w) \big),
\end{align*}
and the global transition function admits the explicit formula 
\begin{align} \label{Eq:GloMod1}
   \bx(\HG) &= \tau \big( \HG, \bu(\HG), \bx(\HG) \big) \nonumber \\
   &= f\big( (\Id \otimes \Win) \bu(\HG) + (\A(\HG) \otimes W) \bx(\HG) \big).
\end{align}
Therefore, by the same argument as in Proposition \ref{Prop:GraphESNContraction}, the global transition function $\tau$ is contractive if we have the following sufficient condition,
\begin{align} \label{Eq:Model1_Cond1}
    \big( \max \left\Vert \A(\HG) \right\Vert \big) \left\Vert W \right\Vert < 1,
\end{align}
where the maximum is taken over all external input data. 

When $g \ne \Id$, the analysis becomes more complicated as there seems to be no simple explicit formula for the global transition function. Nonetheless, under the assumptions of the input null assumption, i.e., $\bu(\HG) = 0$, and differentiability of $f$, we derive a necessary condition for the validity of the iterative scheme. As $\bx(\HG) = 0$ is an equilibrium, we compute the derivative of the global transition function $\tau$. See \cite{Gallicchio2017} for a similar treatment in the standard ESN case. For a hypergraph $\HG$, we have 
\begin{align*}
    \frac{\partial \tau_{v,k}}{\partial \bx(z)_j}|_{\bu = 0, \bx = 0} &= e^T_{v,k} f'(0) W \sum_{h \in H(\HG)} \I_{vh} g'(0) \sum_{ w \in V(\HG)} ( \I_{wh} - \I_{wh}\delta_{vw}) e_{z,j} \\
    &= e^T_{v,k} f'(0) g'(0) W  \sum_{ w \in V(\HG)} \A_{vw} e_{z,j} \\
    &= e^T_{v,k} f'(0) g'(0) W  \A_{vz} e_{z,j}. \\
\end{align*}
Hence, we have 
\[ \nabla \tau = f'(0) g'(0)\A(\HG) \otimes W. \]

In order to have the contraction property, the local stability condition
\[ \rho ( \nabla \tau ) \le \big( \max \rho ( \A(\HG) ) \big) \rho( W ) g'(0) < 1, \]
where $\rho$ denotes the spectral radius, must be satisfied \cite{Gallicchio2017}. If we further require both $f$ and $g$ to be non-negative, $f(0) = g(0) = 0$, and $g$ to be concave, then by repeating the argument in \cite{Higham2021}, the above condition becomes the global stability condition.

At this point, it is natural to ask for a sufficient condition for arbitrary input state $\bu$. Here, we shall give a sufficient condition for the contractivity, even though the bound is not optimal \cite{Gallicchio2010}.
By the direct computation, we have
\begin{align*}
    & \left\Vert \tau(\HG,\bu(\HG),\bx(\HG)) - \tau(\HG,\bu(\HG),\bx'(\HG)) \right\Vert  \\
    &\le \sum_{v \in V(\HG)} \left\Vert W \right\Vert \Bigg\Vert \sum_{h \in H(\HG)} \I_{vh} \Big( g\big(\sum_{ w \in V(\HG)} ( \I_{wh} - \I_{wh}\delta_{vw}) \bx(w) \big) \\
    & \indent \indent \indent \indent \indent \indent \indent \indent - g\big(\sum_{ w \in V(\HG)} ( \I_{wh} - \I_{wh}\delta_{vw}) \bx'(w) \big) \Big) \Bigg\Vert \\
    &\le \sum_{v \in V(\HG)} \left\Vert W \right\Vert \sum_{h \in H(\HG)} \I_{vh} \Bigg\Vert g\big(\sum_{ w \in V(\HG)} ( \I_{wh} - \I_{wh}\delta_{vw}) \bx(w) \big) \\
    & \indent \indent \indent \indent \indent \indent \indent \indent - g\big(\sum_{ w \in V(\HG)} ( \I_{wh} - \I_{wh}\delta_{vw}) \bx'(w) \big)  \Bigg\Vert \\
    &\le \sum_{v \in V(\HG)} \left\Vert W \right\Vert L_g \sum_{h \in H(\HG)} \I_{vh} \left\Vert \sum_{ w \in V(\HG)} ( \I_{wh} - \I_{wh}\delta_{vw}) \big(\bx(w) - \bx'(w)\big) \right\Vert \\
    &\le \sum_{v \in V(\HG)} \left\Vert W \right\Vert L_g \sum_{h \in H(\HG)} \I_{vh} \sum_{ w \in V(\HG)} ( \I_{wh} - \I_{wh}\delta_{vw}) \left\Vert \bx(\HG) - \bx'(\HG) \right\Vert \\
    &\le \sum_{v \in V(\HG)} \left\Vert W \right\Vert L_g \sum_{ w \in V(\HG)} \A_{vw} \left\Vert \bx(\HG) - \bx'(\HG) \right\Vert \\
    &\le \sum_{v,w \in V(\HG)} \A_{vw} \left\Vert W \right\Vert L_g \left\Vert \bx(\HG) - \bx'(\HG) \right\Vert. \\
\end{align*}
Hence, to ensure contractivity, we require 
\begin{align} \label{Eq:Model1_Cond3}
    \big( \max \sum_{v,w \in V(\HG)} \A_{vw} \big) \left\Vert W \right\Vert L_g < 1,
\end{align}
where the maximum is taken over all input hypergraphs in the data set. 

Finally, we remark that if $\HG$ is indeed a graph (i.e., the hypergraph with the constant hyperedge degree $2$) and $g = \Id$, then the HypergraphESN defined by (\ref{Eq:LocTrans2}) reduces to the GraphESN, the condition (\ref{Eq:Model1_Cond1}) coincides with that for GraphESN. The condition (\ref{Eq:Model1_Cond3}) is weaker than the one shown in \cite{Gallicchio2010}, but stronger than (\ref{Eq:Model1_Cond1}).

\subsubsection{Model 2} \label{Chap:Model2}

In Model $2$, we assume that in the aggregation function, a vertex contributes to its incident hyperedges, and thereby replace the \textit{open neighborhood} with \textit{closed neighborhood}. 
The architecture of the model is similar to the previous one with some modifications. The aggregation function $\T_{k,v}: (\R^k)^\star \to \R^{N_R}$ is modified by
\begin{align} \label{Eq:NonLinVerHypInt} 
    \T_{k,v}(h) \coloneqq g\big(\sum_{w \in \ExtNeig(v,h)} \bx(w) \big) = g\big(\sum_{w \in h} \bx(w) \big),
\end{align}
for all $h \in \NeigH(v)$. By staking $\{ \T_{\deg (h),v} (h) \}_{h \in \NeigH(v)}$, we obtain a local hyperedge-based internal state $\bh(\ExtNeig(v)) \in \R^{N_R \cdot \deg(v)}$. However, we notice that as $\T_{k,v}$ is now essentially independent of $v$, we can simply write $\T_k$ for $\T_{k,v}$, and instead of computing the hyperedge-based internal state for each vertex, we have a global hyperedge-based internal state $\bh(\HG) \in \R^{N_R \cdot \left\vert H(\HG) \right\vert}$ by applying $\T$ to each hyperedge. 
Consequently, we have $\bh(\ExtNeig(v)) = \bh(\NeigH(v))$.
The encoding function is again defined vertex-wise by the local transition function $\tau_v$ by 
\begin{align} \label{Eq:LocTrans1}
    \bx(v) &= \tau_v \big( \bu(v), \bh(\NeigH(v)) \big) \nonumber \\
    &= f\big(\Win_v \bu(v) + W_v \bh(\NeigH(v))\big).
\end{align} 

Under the stationary assumption, the global transition function admits the following form 
\begin{align} \label{Eq:Mod2Global}
    \tau(\HG,\bu(\HG),&\bx(\HG)) = \nonumber \\ 
    & f\Big( (\Id \otimes \Win) \bu(\HG) + (\I(\HG) \otimes W) g\big( (\I(\HG)^T \otimes \Id) \bx(\HG) \big) \Big).
\end{align} 
Therefore, if the condition
\[ ( \max\left\Vert \I(\HG) \right\Vert )^2 \left\Vert W \right\Vert L_g < 1 \]
is satisfied, then the transition function $\tau$ is contractive. In particular, if $g = \Id$, then the condition becomes
\[ ( \max\left\Vert \I(\HG)\I(\HG)^T \right\Vert ) \left\Vert W \right\Vert = ( \max\left\Vert \I(\HG) \right\Vert )^2 \left\Vert W \right\Vert < 1. \] 

We observe that in Model $2$, the explicit global formula (\ref{Eq:Mod2Global}) consistently holds for any aggregation function $g$, while Model $1$ requires $g$ to be trivial for Eq. (\ref{Eq:GloMod1}). In addition, under the trivial aggregation function condition, Model $1$ differs from Model $2$ by a degree matrix, and hence, so does GraphESN if $\HG$ is indeed a graph. 

Figure \ref{Fig:2} summarizes the architecture of the Model $2$ of HypergraphESN. The structured input data $\bu(\HG)$ is passed to the structured feature space through the encoding process. In each iteration $t$, the global hyperedge-based internal state $\bh_t(\HG)$ is computed by the non-linear aggregation function and then fed in the reservoir to update $\bx_t(\HG)$. After convergence, the representation $\barbx(\HG)$ of the internal state $\bx(\HG)$ is computed, followed by the readout process. Model $1$ has a similar architecture except that the hyperedge-based internal state should be computed locally.

Before proceeding to the next subsection, we comment on the architecture of the Model $2$ of HypergraphESN. A hypergraph $\HG$ can be represented by a bipartite graph called the incidence graph in which the vertex set is composed of the disjoint union of $V(\HG)$ and $H(\HG)$, and the pair $(v,h) \in V(\HG) \times H(\HG)$ is in the edge set if and only if $v \in h$. The structure of the incidence graph appears in each iteration (see Figure \ref{Fig:2}) and the function $\T$ maps the vectors labeled by one of the bipartitions in the incidence graph to the other. For a similar treatment, see \cite{Yi2020}. 

\begin{figure}[!ht]
    \centering  \includegraphics[width=1\textwidth]{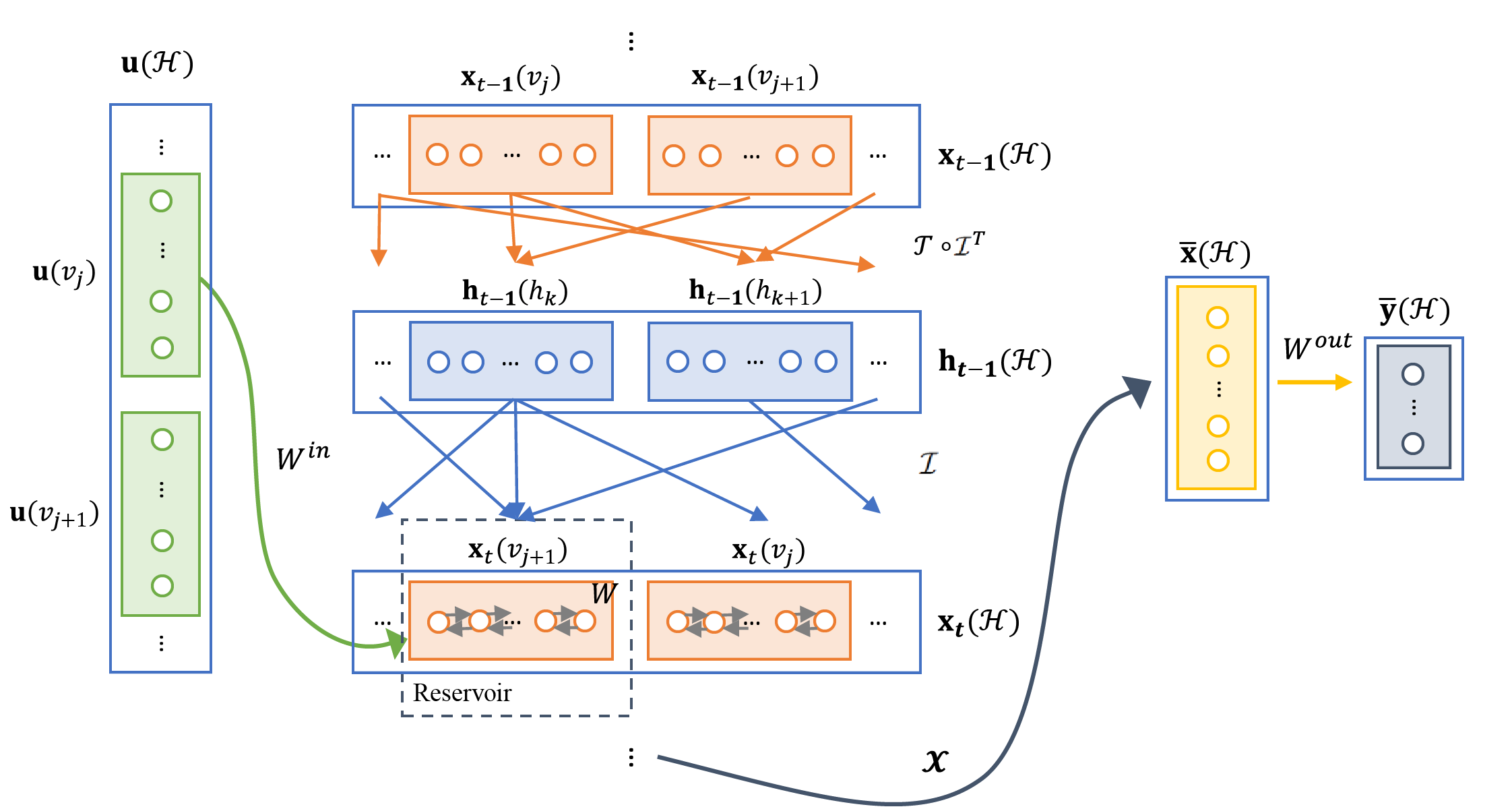}
    \caption{The architecture of HypergraphESN. The green, orange, and blue colored rectangles represent vectors labeled by vertices or hyperedges. The yellow and gray colored rectangles represent the vector representation of the inner state and the output state, respectively.}
    \label{Fig:2}
\end{figure}

\subsection{Computational Complexity}

For each hypergraph $\HG$, the input state $\bu(v)$ is encoded through the local transition function $\tau_v$ in Eq. (\ref{Eq:LocTrans2}) or (\ref{Eq:LocTrans1}). In either case, the computation of a hypergraph-based internal state $\bh(\Neig(v))$ or $\bh(\ExtNeig(v))$ requires $O(\deg(v) \cdot \deg(h) \cdot N_R)$ operations. The application of the matrix multiplication of the reservoir weight matrix $W$ on the hypergraph-based internal state requires $O(\deg(v) \cdot N_R^2)$ operations. If the connectivity of $W$ is further chosen such that each unit is connected to $M$ units on average, then the operations required to update the internal state $\bx_t(v)$ of a vertex $v$ is $O(M \cdot \deg(v) \cdot N_R + \deg(v) \cdot \deg(h) \cdot N_R )$. Therefore, the cost of an update of the global internal state $\bx_t(\HG)$ is
\[ O\Big( \left\vert V(\HG) \right\vert N_R \max\deg(v) \big( M + \max\deg(h) \big) \Big). \]
When the data has upper bounds on the vertex and hyperedge degrees, the computational cost grows linearly in the input network size (i.e., the number of nodes) and the reservoir size. 

\section{Numerical Experiments} \label{Chap:Experi}

We generate data in the SIS setting for classification tasks to analyze how different factors impact the accuracy of the HypergraphESN model, and compare it with the GraphESN.

\subsection{Data Generation}

The generation of the data set consists of $2$ steps. We first generate the hypergraph $\HG$ according to the prescribed hyperedge distribution (i.e., the distribution of hyperedge degrees). Then we implement the hypergraph SIS model (detail in \cite{Higham2021}) and random SIS model. The time series data of the SI labels on each node will be taken as the input vector $\bu(\HG)$.

\subsubsection{Hypergraph Generation}

The hypergraph is generated such that its hyperedge distribution follows a prescribed distribution. For later discussion of the impact of hypergraph structures on the accuracy performance, we consider $3$ types of hypergraph distributions as follows.
\begin{align*}
    p_1(x) = & \frac{1}{C_1} \big( 0.3 \cdot e^{-\frac{(x-2)^2}{3}} + 0.8 \cdot e^{-\frac{(x-5)^2}{5}} + 0.07 \cdot e^{-\frac{(x-30)^2}{3}} \big), \\
    p_2(x) = & \frac{1}{C_2} \big( 0.3 \cdot e^{-\frac{(x-2)^2}{3}} + 0.8 \cdot e^{-\frac{(x-5)^2}{5}} + 0.5 \cdot e^{-\frac{(x-30)^2}{3}} \big), \\
    p_3(x) = & \frac{1}{C_3} \big( 0.3 \cdot e^{-\frac{(x-2)^2}{3}} + 0.8 \cdot e^{-\frac{(x-5)^2}{5}} + 4.0 \cdot e^{-\frac{(x-30)^2}{3}} \big),
\end{align*}
where the domain of the distributions are $\{ 2,3,\dots,40 \}$ and $C_j$'s are normalization constants. The hyperedge distributions are referred to as More Small Hyperedges, Bimodal, More Large Hyperedges, respectively. See Figure \ref{Fig:3}. 

\subsubsection{Sample Data Generation}

Given a hypergraph $\HG$, we generate an SIS epidemic time series data $\{ X_v(t) \in \{0,1\} \}_{t \in \T_\textit{idx}, v \in V(\HG)}$ where $\T_\textit{idx}$ is the index set for time, and extract part of it as the input vector $\bu(\HG)$.

\textbf{Hypergraph SIS.} We implement the algorithm as in \cite{Higham2021} with the concave function $\arctan$, time-step $\Delta t = 0.01$, run-time $t_f = 30$, recovery rate $\delta = 1$, and different infection strength $\beta$. The run-time $t_f$ is chosen such that the later part of the time series data is independent of the initial condition. For each vertex $v$, we pick up last $10$ points (i.e., $N_U = 10$) from the time series data $\{ X_v(t) \}_{t \in \{0,\Delta t,\dots,t_f\}}$ as input data $\bu(v) \in \R^{10}$.

\textbf{Random SIS.} Each vertex does not interact with other vertices (i.e., the underlying structure is fully discrete). For each vertex $v$, the transition rate from susceptible to infected follows the Poisson process with parameter $\delta = 1$ and the opposite with parameter $\beta'$.

\subsection{Experiment Setup} \label{Chap:ExpSetup}

We consider two classes of tasks. Class A consists of $55 \%$ and $45 \%$ of samples whose input state $\bu$ records the time series data of hypergraph SIS process with $\beta = 1$ and random SIS with $\beta'$, respectively. Here, $\beta'$ is chosen such that the expected ratio of infected nodes is the same as that of the Hypergraph SIS process. Class B consists of $55 \%$ and $45 \%$ of samples whose input state $\bu$ records the time series data of hypergraph SIS process with $\beta = 0.085$ and with $\beta = 0.115$, respectively. For each class, we consider $3$ types of hyperedge distributions $p_j$. 
For each $p_j$, we generate $1000$ data sets, each hypergraph consists of $n$ nodes and $m$ hyperedges such that $<\deg(v)> = 3$. In short, from the viewpoint of population-level models, Class A can be understood as the same severity of the epidemic but on different types of structures, while Class B is the opposite.

In section \ref{Chap:HESN_vs_ESN}, we study how the understanding of higher-order interactions among nodes affects the accuracy performance by implementing ESNs with different underlying structures as input.
The reservoir size $N_R$ is contained in $\{2^k: k = 0,\dots, 6\}$ and the network size is set to be $80$. We (\romannum{1}) run the HypergraphESN with full hypergraph structures, (\romannum{2}) randomly replace $50\%$ of hyperedges with complete graphs, remove the identical edges, and then run the HypergraphESN. The resulting hypergraph is called a partial clique hypergraph $\HG_{50\%}$ of the hypergraph $\HG$ and its curve is labeled $50\%$ \textit{HypergraphESN}. In particular, we implement the Model $2$ version due to its compact form Eq. (\ref{Eq:Mod2Global}), with $f = \tanh$ and $g = \Id$. 
The input weight matrix $\Win$ is a random matrix of i.i.d. values sampled from standard uniform distribution, and the reservoir weight matrix $W$ is randomly generated from standard uniform distribution with $M = 5$ sparseness and then normalized such that $\left\Vert W \right\Vert = 0.9 \cdot ( \max \left\Vert \I(\HG) \right\Vert^{2} )^{-1}$. The output function is trained by the ridge regression with regularization $\lambda_\textit{ridge} \in \{ 0, 10^{-5}, 5\times 10^{-5},\dots,10^{-1}, 5\times 10^{-1} \}$ \cite{Micheli2022,Wang2022}. Moreover, we (\romannum{3}) replace hypergraphs $\HG$ with their the clique expansion $\G_\HG$ (i.e., replacing all hyperedges with edges) and run the GraphESN with the same setup as the HypergraphESN, except that the reservoir weight matrix $W$ now satisfies $\left\Vert W \right\Vert = 0.9 \cdot ( \max \left\Vert \A(\G_\HG) \right\Vert )^{-1}$. 
As the ESN models depend on the random weights initialization, we evaluate the accuracy of the classification tasks by conducting a bootstrap analysis as follows. We randomly draw (with replacement) $200$ samples out of the original $1000$ data sets, implement the ESN with a $90\%-10\%$ training/test split, and repeat this bootstrap re-sampling process $200$ times to estimate the accuracy.
Finally, as a reminder, we may also regard $\G_\HG$, $\HG_{50\%}$, and $\HG$ as poorly, partially, and fully identifying the higher-order interactions in a network.

In section \ref{Chap:NR_vs_n}, we study the dependence of the accuracy performance of Hypergraphs on reservoir size and network size. We apply the same setup as in section \ref{Chap:HESN_vs_ESN} but with varying reservoir size $N_R \in \{2^k: k = 0,\dots,9 \}$ and network sizes $n \in \{ 50, 60,\dots, 110 \}$. For the sake of computational resources, we only do the bootstrap re-sampling process $20$ times.

In section \ref{Chap:Nonlinearity}, we study the influence of the nonlinear vertex-hyperedge interaction on the accuracy. The same setup as in section \ref{Chap:HESN_vs_ESN} is applied except four different aggregation functions are considered: the identity function and functions of the type $g_{(a,b)}$, which are piecewise linear functions passing through the origin such that in the intervals $(-\infty,0]$, $[0,a)$, $[a,b)$, and $[b,\infty)$ the slopes are $0$, $0.1$, $1$, and $3$, respectively.

\subsection{Performance and Discussion}

\subsubsection{HypergraphESN vs GraphESN} \label{Chap:HESN_vs_ESN}

Figure \ref{Fig:3} illustrates how reservoir size $N_R$ and the understanding of hypergraph structure in a network contribute to increased classification accuracy. We observed two general trends: (\romannum{1}) the accuracy of HypergraphESN, $50 \%$ HypergraphESN, and GraphESN shows improvement as $N_R$ increases, and (\romannum{2}) accuracy tends to rise as more hyperedges are identified, provided that $N_R$ is sufficiently large. Furthermore, we note that even if the dynamics of epidemic spread on small and large hyperedges are different, the effectiveness of HypergraphESN is consistent.

The discrepancy in accuracy performance between Class A and Class B tasks stems from the fact that Class A tasks require a great number of units in the hidden layer to achieve higher accuracy, whereas Class B task accomplishes a decent accuracy with considerably fewer units but quickly approaches its limitation. We believe that since the classification of input states from infection strength $\beta = 0.085$ and $0.115$ can be achieved with just one or two features from the data (e.g., the mean ratio of infected nodes), the accuracy performance experiences a rapid increase even with a limited number of units in the hidden layer. However, due to the insufficient length of input states extracted from the time series data of the hypergraph SIS model, increasing the reservoir size further would not lead to an improvement in classification accuracy.

Moreover, in Class A classification tasks, compared to the GraphESN, when $50 \%$ of hyperedges are identified, the accuracy improves more notably if more giant components (i.e., large hyperedges) are presented in the underlying hypergraph structure. This phenomenon may be attributed to the nature of hypergraphs and SIS models. When the small hyperedges prevail, the non-linear concave arctangent function applied in the SIS model does not play a vital role. As a consequence, intuitively speaking, the SIS process on the partial clique hypergraph may lie in the middle of that on the hypergraph and the clique graph. 
However, the arctangent concavity influences the SIS process through the large hyperedges, and therefore, even if merely half of the hyperedges are identified, $50\%$ HypergraphESN may improve the accuracy performance more significantly. In Class B tasks, the accuracy of $50\%$ HypergraphESN is comparable to that of HypergraphESN, and accuracy improvements do not depend on hyperedge distributions. This may be ascribed to the fact that the Class B tasks can be distinguished with few features. 

\begin{figure}[ht]
    \centering  \includegraphics[width=1\textwidth]{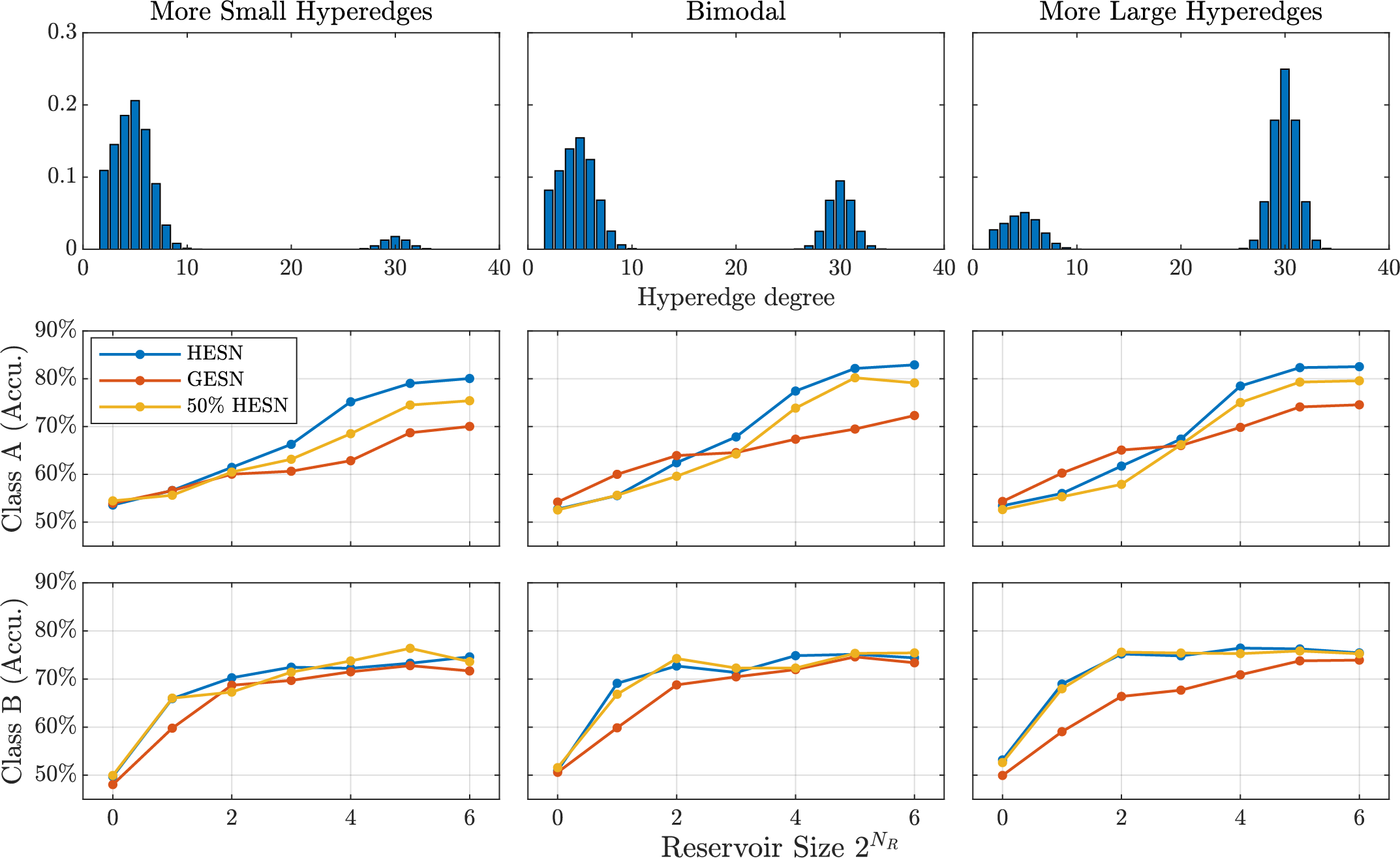}
    \caption{The accuracy performance under different hyperedge distributions. The sub-figures in the first row, from left to right, are the hyperedge distributions corresponding to \textit{More Small hyperedges}, \textit{Bimodal}, and \textit{More Large hyperedges}. Those in the second and the third rows, from left to right, show the mean accuracy performance of HypergraphESN (\textit{HESN}), $50\%$ HypergraphESN (\textit{50\% HESN}), and GraphESN (\textit{GESN}) for Class A and in Class B tasks, respectively. }
    \label{Fig:3}
\end{figure}

\subsubsection{Reservoir Size vs Network Size} \label{Chap:NR_vs_n}

Figure \ref{Fig:4} shows how network size affects the accuracy. In both Class A and Class B tasks, a higher accuracy performance can be achieved for larger networks provided that reservoir size is sufficiently large. This trend arises from the larger network's capacity to offer more information for HypergraphESN to effectively classify tasks. Moreover, we notice that HypergraphESN demonstrates a better performance for networks featuring more giant components. This phenomenon is again attributed to the nature of the hypergraph SIS process. 
Under the constant vertex degree condition, larger hyperedges not only enrich the connectivity of the network but also significantly change the dynamics of the epidemic through the arctangent concavity. Consequently, the networks with more giant components exhibit a more complex and diverse behavior compared to the networks on which small hyperedges dominate, and thereby provide HypergraphESN with richer information for the classification tasks.

\begin{figure}[ht]
    \centering  \includegraphics[width=1\textwidth]{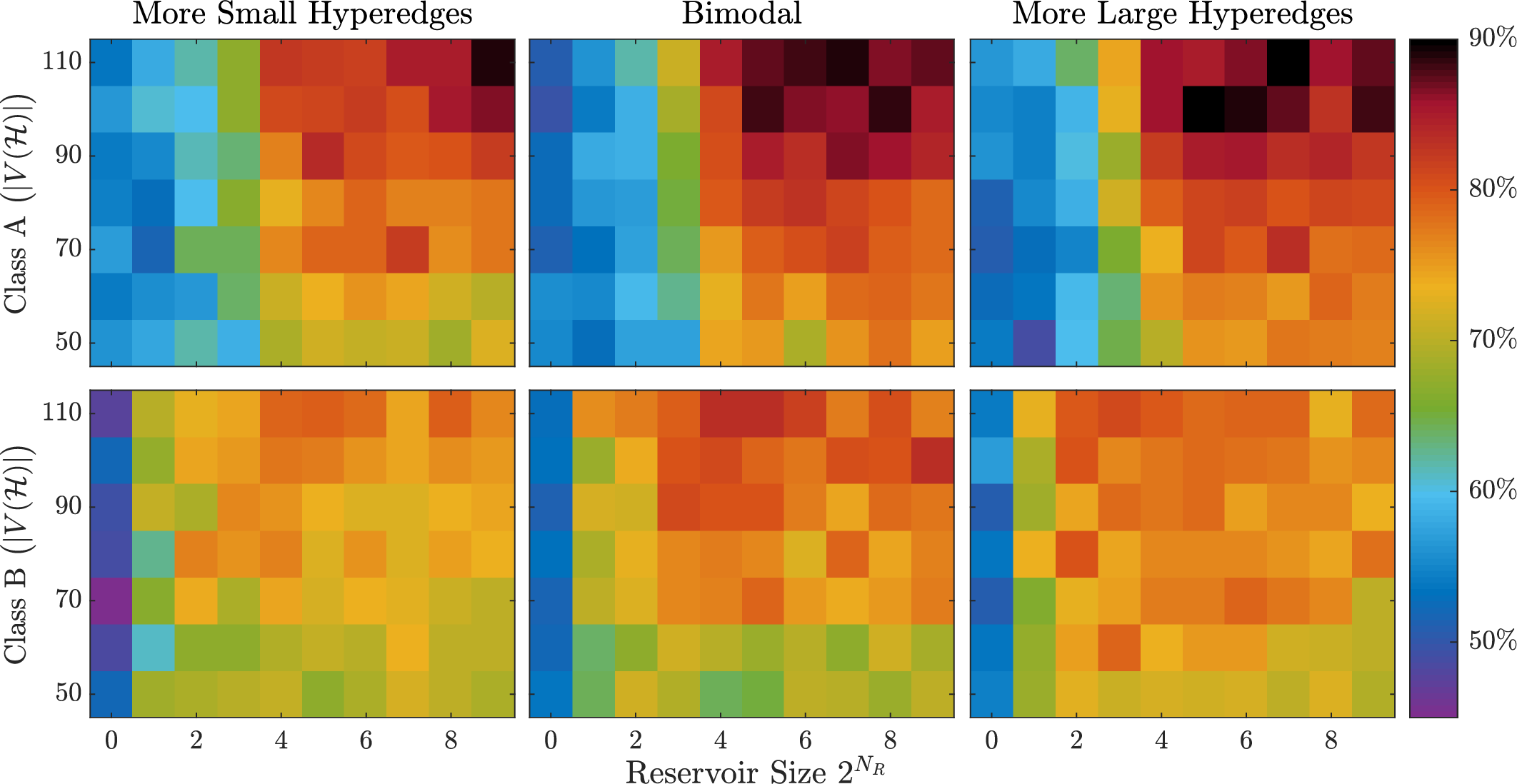}
    \caption{The accuracy performance of HypergraphESN. The $y$-axis is the network sizes and the colorbar indicates the mean accuracy.}
    \label{Fig:4}
\end{figure}

\subsubsection{Non-linearity of Vertex-hypergraph Interaction} \label{Chap:Nonlinearity}

Since the HypergraphESN allows the non-linearity vertex-hyperedge interaction through the aggregation function, we focus on the Bimodal case where small and large hyperedges compete. As mentioned in section \ref{Chap:ExpSetup}, we consider $\Id$, $g_{(3,6)}$, $g_{(7,10)}$, and $g_{(11,14)}$. For instance, 
\begin{align} \label{Eq:NonLinEx}
    g_{(3,6)}(x) = 
    \begin{cases}
        0, & x < 0 \\
        0.1 x, & 0 \le x < 3 \\
        x - 2.7, & 3 \le x < 6 \\
        3x - 14.7, & 10 \le x. \\
    \end{cases}
\end{align}
The essential idea behind the piecewise linear functions $g_{(a,b)}$ is to distinguish the influence of small and large hyperedges by an artificial threshold which may highly depend on the tasks at hand. In this study, as the epidemic tends to spread more efficiently through large hyperedges, the node $v$ belonging to a large hypergraph has a higher chance of being infected. Consequently, each entry in $\bx_1(v) = \tau_v(\bu(v))$ should be higher, and vice versa. 
Figure \ref{Fig:5} shows that the HypergraphESN with artificially-designed functions may experience a more rapid increase in accuracy as the reservoir increases, and show a comparable accuracy as HypergraphESN with the trivial aggregation function $\Id$, if the correct threshold to distinguish the hyperedges degrees is identified. This observation implies that with a carefully designed vertex-hypergraph interaction, the reservoir can yield sufficiently diverse dynamics, even with a smaller reservoir size, consistent with the remark made in \cite{GALLICCHIO2011}. In Class A tasks, Eq. (\ref{Eq:NonLinEx}) may correctly identify the small and large hyperedges in the reservoir computing while other piecewise linear functions do not. In Class B tasks, none of the piecewise linear functions improve the HypergraphESN performance notably compared to the identity map.

\begin{figure}[ht]
    \centering  \includegraphics[width=0.5\textwidth]{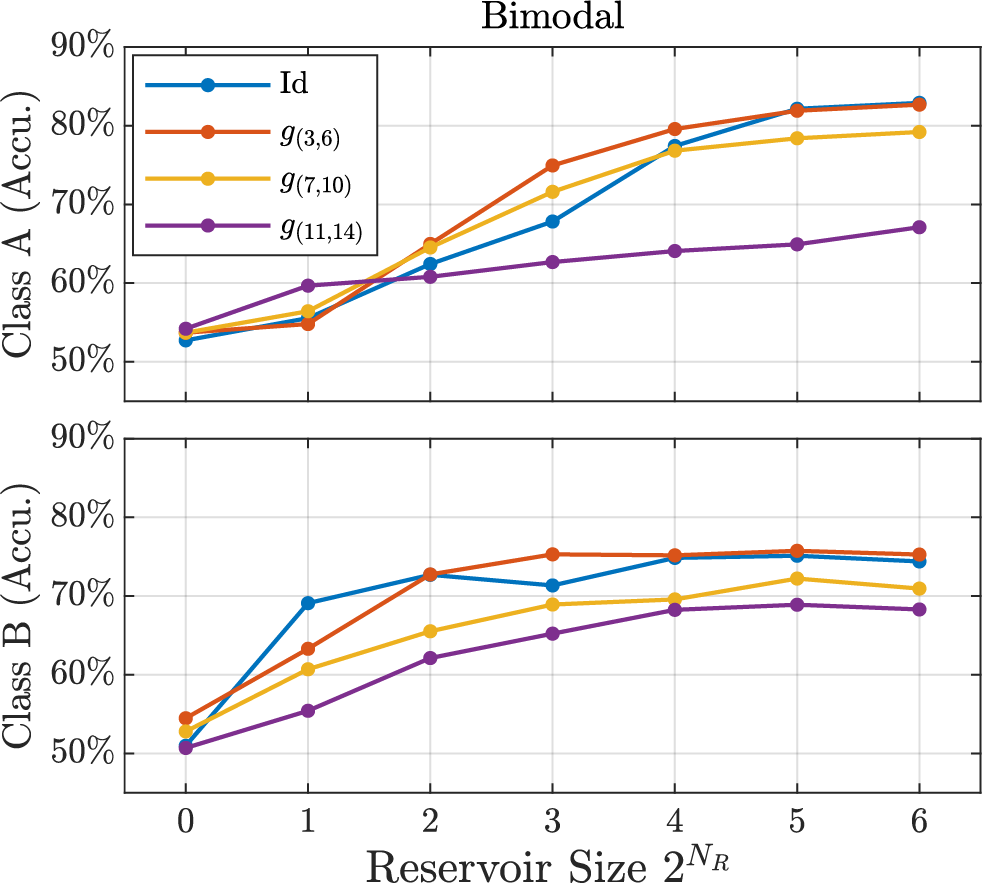}
    \caption{The accuracy performance of HypergraphESN with different choices of aggregation functions. The Bimodal hyperedge distribution is applied.}
    \label{Fig:5}
\end{figure}

\section{Conclusion}

We have introduced an extension of the GraphESN to the hypergraph domain, and studied the theoretical convergence conditions for the encoding process. As a variant of the ESN, HypergraphESN offers an efficient approach for processing hypergraph-structured data. Moreover, the computational complexity can be designed to exhibit linear growth in both the input network size and the reservoir size.

Our numerical experiments for binary classification have shown that for hypergraph-structured data, the HypergraphESN achieves accuracy levels comparable or even higher to those attained by GraphESN. The accuracy further improves with the identification of more higher-order relationships (i.e., hyperedges) in the network. 
Furthermore, it is worth noting that in the case where the optimal accuracy performance of HypergraphESN and GraphESN in Class B tasks are not significantly different, the HypergraphESN exhibits a more rapid learning capacity as the reservoir size increases. In addition, due to the distinguished dynamics of epidemics on giant components, the HypergraphESN exhibits better performance on networks with larger hyperedges. These findings suggest a promising potential for the HypergraphESN to enhance the performance of graph-based models on real-world data, both in terms of accuracy and computational efficiency, provided that a suitable hypergraph representation of the network can be identified.

We also considered the non-linear vertex-hyperedge interaction described by the aggregation function, which can be designed to introduce bias on a specific range of hypergraph degrees and thereby enrich the reservoir dynamics even with a limited number of hidden units.  
In this article, we applied the piecewise linear functions to artificially distinguish the small and large hyperedge in the reservoir, and revealed the potential to improve the learning capability of HypergraphESN though a suitable candidate of the aggregation function may highly depend on the tasks.

Overall, the HypergraphESN serves as a more practical and flexible tool to study structured data with higher-order relationships. Though the presented results are promising, more directions deserve to be explored in the future: the non-linearity interaction between vertices and hyperedges, a simpler architecture design, the trade-off between model performance and computational resources, the robustness against the real-world data, etc. We will examine HypergraphESN from those aspects in the subsequent works. 

\section*{Declarations of Competing Interests}
The author has no relevant financial or non-financial interests to disclose.

\section*{Declaration of Generative AI and AI-assisted Technologies in the Writing Process}
During the preparation of this work, the author used Chap-GPT 3.5 in order to improve the readability and language. After using this tool/service, the author reviewed and edited the content as needed and takes full responsibility for the content of the publication.

\section*{Acknowledgments}
The author would like to express his gratitude to Professor Ando H. (WPI-AIMR) for his suggestion on this article.

\addcontentsline{toc}{chapter}{Bibliography} 






\begin{thebibliography}{30}




\bibitem{Bandyopadhyay2020}
Bandyopadhyay S., Das K., \& Murty MN. 
Line hypergraph convolution network: Applying graph convolution for hypergraphs. 
\textit{arXiv preprint arXiv: \textsl{2002.03392}}. 

\bibitem{Bretto2013}
{Bretto, A.}, 
(2013). 
\textit{Hypergraph theory: An introduction.}, 
Mathematical Engineering, 
Springer Cham. \\ 
https://doi.org/10.1007/978-3-319-00080-0 


\bibitem{Chatzis2011}
Chatzis, SP. \& Demiris, Y. 
(2011). 
Echo State Gaussian Process. 
\textit{IEEE Transactions on Neural Networks}, 
\textsl{22}, 
1435--1445. \\ 
doi: 10.1109/TNN.2011.2162109 

\bibitem{CHUNG1997}
{Chung F.R.K.}, 
(1997). 
\textit{Spectral Graph Theory}, 
Conference Board of Mathematical Sciences, 
American Mathematical Society. 


\bibitem{Dai2009}
Dai, J., Venayagamoorthy, G. K. \& Harley, R. G. 
(2009). 
\textit{An Introduction to the Echo State Network and its Applications in Power System}. 
2009 15th International Conference on Intelligent System Applications to Power Systems, Curitiba, Brazil, 
pp. 1--7. \\ 
doi: 10.1109/ISAP.2009.5352913 


\bibitem{Feng2019}
Feng, Y., You, H., Zhang, Z., Ji, R. \& Gao, Y. 
(2019). 
Hypergraph Neural Networks. 
\textit{Proceedings of the AAAI Conference on Artificial Intelligence}, 
\textsl{33}, 
3558--3565. \\ 
https://doi.org/10.1609/aaai.v33i01.33013558 

\bibitem{Gallicchio2010}
{Gallicchio, C., \& Micheli, A.} 
(2010). 
\textit{Graph Echo State Networks}. 
The 2010 International Joint Conference on Neural Networks (IJCNN), Barcelona, Spain, 
pp. 1--8. \\ 
doi: 10.1109/IJCNN.2010.5596796 

\bibitem{GALLICCHIO2011}
Gallicchio, C., \& Micheli, A. 
(2011). 
Architectural and Markovian factors of echo state networks. 
\textit{Neural Networks}, 
\textsl{24}. \\ 
https://doi.org/10.1016/j.neunet.2011.02.002 

\bibitem{Gallicchio2017}
Gallicchio, C., \& Micheli, A. 
(2017). 
Echo State Property of Deep Reservoir Computing Networks. 
\textit{Cognitive Computation}, 
\textsl{9}, 
337--350. \\ 
https://doi.org/10.1007/s12559-017-9461-9 

\bibitem{GALLICCHIO2023}
Gallicchio, c. 
(2023). 
Euler State Networks: Non-dissipative Reservoir Computing. 
\textit{arXiv preprint arXiv: \textsl{2203.09382}}. \\ 


\bibitem{Higham2021}
Higham D. J. \& Henry-Louis d. K. 
(2021). 
Epidemics on hypergraphs: spectral thresholds for extinction. 
\textit{Processings of the Royal Society A}, 
\textsl{477}. \\ 
http://doi.org/10.1098/rspa.2021.0232 

\bibitem{Hirono2021}
Hirono, Y., Okada, T., Miyazaki, H., \& Hidaka, Y. 
(2021). 
Structural reduction of chemical reaction networks based on topology. 
\textit{Physical Review Research}, 
\textsl{3}, 
043123. 
https://link.aps.org/doi/10.1103/PhysRevResearch.3.043123 

\bibitem{Huang2019}
Huang, J., Cao, Y., Xiong, C., \& Zhang, H.T. 
(2019). 
An Echo State Gaussian Process-Based Nonlinear Model Predictive Control for Pneumatic Muscle Actuators. 
\textit{IEEE Transactions on Automation Science and Engineering}, 
\textsl{16}, 
1071--1084. \\ 
doi: 10.1109/TASE.2018.2867939 


\bibitem{Jaeger2001}
Jaeger, H. 
(2001). 
The “echo state” approach to analysing and training recurrent neural networks-with an erratum note. 
\textit{Bonn, Germany: German National Research Center for Information Technology GMD Technical Report}, 
\textsl{148}. \\ 
https://api.semanticscholar.org/CorpusID:15467150 

\bibitem{JAEGER2007}
Jaeger, H., Lukoševičius, M., Popovici, D., \& Siewert, U. 
(2007). 
Optimization and applications of echo state networks with leaky-integrator neurons. 
\textit{Neural Networks}, 
\textsl{20}, 
335--352. \\ 
https://doi.org/10.1016/j.neunet.2007.04.016 

\bibitem{Jia2021}
Jia, R., Zhou, X., Dong, L., \& Pan, S. 
(2021). 
\textit{Hypergraph Convolutional Network for Group Recommendation}. 
2021 IEEE International Conference on Data Mining (ICDM), Auckland, New Zealand, 2021,  
pp. 260--269. \\ 
doi: 10.1109/ICDM51629.2021.00036 

\bibitem{Jiang2019}
Jiang, J. W., Wei, Y. X., Feng, Y. F., Cao, J. X. \& Gao, Y. 
(2019). 
\textit{Dynamic Hypergraph Neural Networks}. 
Proceedings of the Twenty-Eighth International Joint Conference on Artificial Intelligence Main track (IJCAI-19), Macao, China, pp. 2635--2641. \\ 
https://doi.org/10.24963/ijcai.2019/366 


\bibitem{Konstantinova2001}
Konstantinova E. V., Skorobogatov, V. A. 
(2001). 
Application of hypergraph theory in chemistry. 
\textit{Discrete Mathematics}, 
\textsl{235}, 
365--383. \\ 
https://doi.org/10.1016/S0012-365X(00)00290-9. 


\bibitem{Lun2019}
Lun, S. X., Hu, H. F., Yao, X. S. 
(2019). 
The modified sufficient conditions for echo state property and parameter optimization of leaky integrator echo state network. 
\textit{Applied Soft Computing}, 
\textsl{77}, 
750--760. \\ 
https://doi.org/10.1016/j.asoc.2019.02.005 


\bibitem{Micheli2022}
Micheli, A., Tortorella D. 
(2022). 
Discrete-time dynamic graph echo state networks. 
\textit{Neurocomputing}, 
\textsl{496}, 
85--95. \\ 
https://doi.org/10.1016/j.neucom.2022.05.001. 

\bibitem{Mulas2022}
Mulas, R., Horak, D., \& Jost, J. 
(2022). 
Graphs, Simplicial Complexes and Hypergraphs: Spectral Theory and Topology. In Battiston, F., \& Petri, G. (Eds.), \textit{Higher-Order Systems}, (pp. 1--58). \\
https://doi.org/10.1007/978-3-030-91374-8\-1 


\bibitem{Nowzari2016}
Nowzari, C., Preciado, V. M., \& Pappas, G. J. 
(2016). 
Analysis and Control of Epidemics: A Survey of Spreading Processes on Complex Networks. 
\textit{IEEE Control Systems Magazine}, 
\textsl{36}, 
26--46. 
doi: 10.1109/MCS.2015.2495000 

\bibitem{Reiser2022}
Reiser, P., Neubert, M., Eberhard, A. et al. 
(2022). 
Graph neural networks for materials science and chemistry. 
\textit{Communications Materials}, 
\textsl{3}, 
Article 93. \\ 
https://doi.org/10.1038/s43246-022-00315-6 


\bibitem{Wang2022}
Wang, L., Zhao, J. \& Zhang, Y. X. 
(2022). 
\textit{Ridge-regression Echo State Network for Effluent Ammonia Nitrogen Prediction}.
2022 China Automation Congress (CAC), Xiamen, China, pp. 6509--6512. \\
doi: 10.1109/CAC57257.2022.10055863 

\bibitem{Wilsonm2020}
Wilson, N., Corbett, S., \& Tovey, E. 
(2020). 
Airborne transmission of covid-19. 
\textit{British Medical Journal}, 
\textsl{370}. \\ 
https://www.bmj.com/content/370/bmj.m3206 


\bibitem{Yadati2019}
Yadati, N., Nimishakavi, M., Yadav, P., Nitin, V., Louis, A., \& Talukdar, P. 
(2019). 
HyperGCN: A new method for training graph convolutional networks on hypergraphs. 
\textit{Proceedings of the 33rd International Conference on Neural Information Processing Systems}, 
Article 135, 
1511--1522. 


\bibitem{Yi2020}
Yi, J., \& Park, J. 
(2020). 
\textit{Hypergraph Convolutional Recurrent Neural Network}. 
Proceedings of the 26th ACM SIGKDD International Conference on Knowledge Discovery \& Data Mining, New York, USA, pp. 3366--3376. \\ 
https://doi.org/10.1145/3394486.3403389 

\bibitem{Yu2019}
Yu, Y., Si, X., Hu, C., Zhang, J. 
(2019). 
A Review of Recurrent Neural Networks: LSTM Cells and Network Architectures. 
\textit{Neural Computation}, 
\textsl{31}, 
1235--1270. \\ 
https://doi.org/10.1162/neco\_a\_01199 


\bibitem{Zhang2016}
Zhang, Z. K., Liu, C., Zhan, X. X., Lu, X., Zhang, C. X., Zhang, Y. C. 
(2016). 
Dynamics of Information Diffusion and its Applications on Complex Networks. 
\textit{Physics Reports}, 
\textsl{651}, 
1--34. \\ 
https://doi.org/10.1016/j.physrep.2016.07.002 

\bibitem{Zhou2007}
Zhou, D. Y., Huang, J. Y., \& Sch\"{o}lkopf, B. 
(2007). 
Learning with Hypergraphs: Clustering, Classification, and Embedding. 
In Sch{\"o}lkopf, B., Platt, J., \& Hofmann, T. (Eds.), \textit{Advances in Neural Information Processing Systems \textsl{19}: Proceedings of the \textsl{2006} conference}, MIT Press. \\ 
https://doi.org/10.7551/mitpress/7503.003.0205 

\end{thebibliography}



\end{document}